\documentclass[letterpaper]{article} 
\usepackage{aaai2026}  
\usepackage{times}  
\usepackage{helvet}  
\usepackage{courier}  
\usepackage[hyphens]{url}  
\usepackage{graphicx} 
\usepackage{natbib}  
\usepackage{caption} 
\usepackage{algorithm}

\usepackage{newfloat}
\usepackage{listings}
\DeclareCaptionStyle{ruled}{labelfont=normalfont,labelsep=colon,strut=off} 
\floatstyle{ruled}
\newfloat{listing}{tb}{lst}{}
\floatname{listing}{Listing}
\usepackage[noend]{algpseudocode}

\usepackage{amsthm}
\usepackage{amsmath} 
\usepackage{booktabs}
\usepackage{multirow}
\usepackage{amssymb}
\usepackage{subcaption}
\usepackage{accents}
\usepackage{booktabs}
\usepackage{siunitx}
\usepackage{todonotes}

\newenvironment{proofsketch}
  {\proof}
  {\endproof}

\newtheorem{theorem}{Theorem}
\newtheorem{lemma}{Lemma}
\newtheorem{definition}{Definition}

\DeclareMathOperator*{\argmin}{argmin}

\newcommand{\blue}[1]{{#1}}
\newcommand{\red}[1]{}

\title{Front-to-Attractors: Modifying the Front-to-Front Heuristic in Bidirectional Search}
\author{
    Alvin Zou\textsuperscript{\rm 1},
    Muhammad Suhail Saleem\textsuperscript{\rm 1},
    Maxim Likhachev\textsuperscript{\rm 1}
}
\affiliations{
    \textsuperscript{\rm 1}Carnegie Mellon University\\
    azou@andrew.cmu.edu, msaleem2@andrew.cmu.edu,
    maxim@cs.cmu.edu
}

\usepackage{bibentry}

\begin{document}

\maketitle

\begin{abstract}
Heuristics play a central role in the performance of bidirectional search algorithms, which commonly rely on two main classes. \textit{Front-to-end} (F2E) heuristics estimate the distance from a state $s$ to the target of the search (the $\mathit{goal}$ for forward search or the $\mathit{start}$ for backward search). In contrast, \textit{front-to-front} (F2F) heuristics estimate the distance from $s$ to the opposite search frontier using a pairwise function $h(s, s')$, where $s'$ ranges over frontier states. Although F2F heuristics are typically more informative and therefore reduce the number of node expansions, their reliance on extensive pairwise evaluations incurs substantial computational overhead. To address this limitation, we introduce a new heuristic class, \textit{front-to-attractors} (F2A), that preserves much of the informativeness of F2F while dramatically reducing its computational cost. Rather than evaluating distances to all states on the opposite frontier, F2A estimates the distance from $s$ to a small, dynamically maintained set of \textit{attractors} in the opposite search direction. These attractors serve as a surrogate for the full frontier, enabling rich heuristic guidance at a fraction of the computational expense while maintaining the optimality guarantees offered by F2F. We evaluate F2A across multiple domains and show that it reduces the number of pairwise evaluations by up to $11.2\times$ compared to F2F, while achieving $4.8\times$ fewer node expansions than F2E on average.
\end{abstract}


\section{Introduction}

Bidirectional heuristic search (Bi-HS) aims to find the least-cost path between two states by running a forward and backward search simultaneously, reducing effective search depth as the two frontiers meet. A central design choice in Bi-HS is the heuristic used to guide these searches. Most existing algorithms \red{rely on}\blue{use} \textit{front-to-end} (F2E) heuristics, which estimate the \red{distance (i.e., cost)}\blue{cost} from a state to the target of the current search direction (\red{the }\textit{goal} for forward search or \red{the }\textit{start} for backward search). Although inexpensive, F2E heuristics offer weak guidance because they treat the two searches independently and ignore the progress of the opposite frontier (i.e., Open list). In contrast, \textit{front-to-front} (F2F) heuristics estimate the distance from a state to the opposite frontier using a pairwise function $h(s, s')$, where $s'$ ranges over \blue{the }states in the \red{frontier of the opposite direction}\blue{opposite frontier}. F2F heuristics are more informative \cite{siag2023comparing} and typically require fewer state expansions to find a solution. Their drawback is computational cost: evaluating the heuristic of a state in one direction requires pairwise \red{heuristic }evaluations between the state and all states in the opposite frontier.

Due to this overhead leading to long planning times, most work on improving Bi-HS focuses on F2E methods, proposing enhanced expansion rules, stronger termination conditions, and effective pruning strategies \cite{BAE, holte2017mm, shaham2017minimal, eckerle2017sufficient, NBS, DIBBS, DVCBS, MEET}. Only a few approaches attempt to exploit or refine F2F heuristics \cite{felner2010single, lippi2012efficient, lippi2016optimally, mayer2019front, alcazar2021consistent}. Although F2F-based algorithms offer strong potential as their heuristics are more informative and often reduce the number of state expansions, they remain significantly slower in practice because heuristic evaluation scales quadratically with frontier size \cite{siag2023comparing}. Techniques that try to represent the frontier with a smaller set exist, but give suboptimal solutions \cite{anchor_search}.

To overcome the overhead associated with F2F heuristics, this manuscript introduces a new heuristic class, front-to-attractors (F2A), that retains much of the informativeness of F2F while greatly reducing its computational cost. The approach leverages the attractor representation introduced in \citet{zou2025attractor}. Instead of comparing $s$ against all states on the opposite frontier, F2A evaluates distances to a small, dynamically maintained set of \textit{attractors}. These attractors serve as a compact surrogate for the full frontier, enabling rich heuristic guidance at a fraction of the computational expense while preserving the optimality guarantees of F2F. We evaluate F2A across multiple domains, including 2D grid pathfinding, Sliding Tiles, and Pancake puzzle, and show that it reduces the number of pairwise heuristic evaluations by up to $11.2\times$ compared to F2F, while achieving $4.8\times$ fewer node expansions than F2E on average.


\section{Background}
A \red{bidirectional heuristic search (Bi-HS)}\blue{Bi-HS} problem instance $I=(G,start,goal,h_F,h_B)$ consists of a graph $G =(V,E)$, a $start$ state, a $goal$ state, and a forward (F) and backward (B) heuristic $h_F,h_B$. The cost of the shortest path between state $s$ and \red{state }$s'$ in $G$ is denoted by $c(s, s')$. The objective is to find a path from $start$ to $goal$ with optimal cost \red{$c(start,goal)$}\blue{$C^*$}. Bi-HS algorithms perform two searches simultaneously: a forward search from $start$ and a backward search from $goal$. Each search direction $D \in \{F,B\}$ maintains its own Open list, $Open_D$. For a state $s$, $g_D(s),h_D(s),f_D(s)$ denote \red{the}\blue{its} $g$-, $h$-, and $f$-values \red{of $s$ }in direction $D$. The opposite direction is denoted by $OD$.


\textit{front-to-end} (F2E) heuristics use a forward heuristic $h_F: \mathcal{S} \rightarrow \mathbb{R}_{\ge 0}$ that estimates the cost from any state to $goal$ and a backward heuristic $h_B: \mathcal{S}  \rightarrow \mathbb{R}_{\ge 0}$ that estimates the cost from any state to $start$. The $f$-value is computed as $f_D(s) = g_D(s) + h_D(s)$.\red{In contrast,} \textit{front-to-front} (F2F) heuristics \cite{BHFFA} use a pairwise heuristic $h: \mathcal{S} \times \mathcal{S} \rightarrow \mathbb{R}_{\ge 0}$ that estimates the cost between any pair of states. The F2F heuristic of a state $s$ in \red{search }direction $D$ is computed by estimating the cost between $s$ and the opposite frontier. The $f$-value is \red{computed as }$f_D(s) = g_D(s) + \min_{s' \in Open_{OD}}(h(s,s')+g_{OD}(s'))$. \blue{We use the same definitions for bi-admissibility and bi-consistency as \cite{eckerle2017sufficient}.}

\begin{figure}
    \centering

    \includegraphics[width=0.5\linewidth]{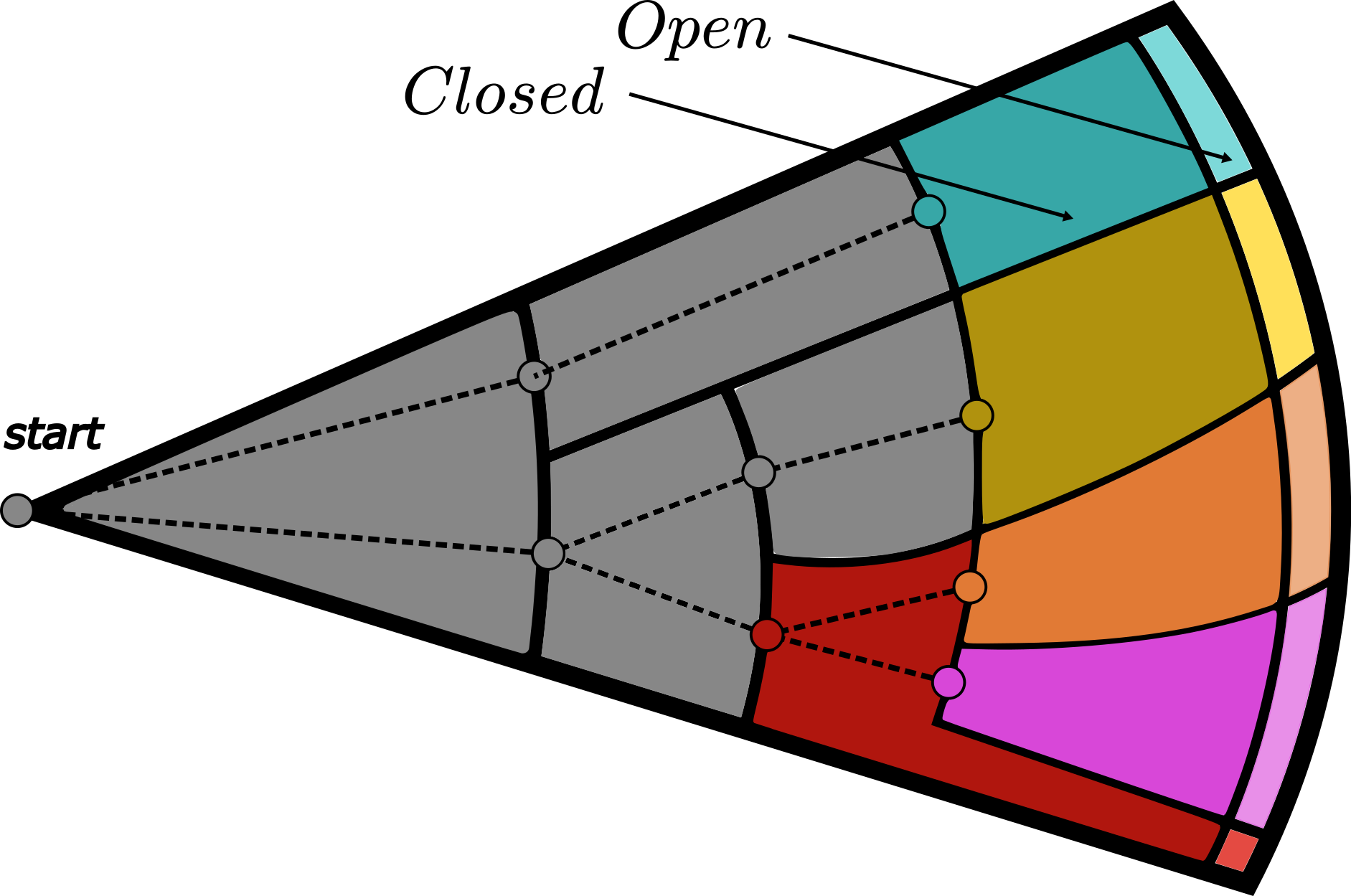}
    \caption{Attractors are shown as circles, with associated states in matching colors. The Open list is shown in lighter shades, and the remaining states are in the Closed list. Attractors with no associated frontier states are shown in gray.}
    \label{fig:attractor_3} \vspace{-0.3cm}
\end{figure}


\paragraph{Attractors}

Attractors \cite{zou2025attractor} were introduced as a memory-efficient alternative to storing the full Closed list in best-first search. They represent the Closed list sparsely by storing only an intelligently selected small set of states and using them as anchors for solution reconstruction. During search, every state added to the Open list is assigned to an attractor, and each attractor (except $start$) is linked to a parent attractor, forming a tree that serves as a compact surrogate for the full Closed list\red{ (Fig.~1)}. \blue{Fig.~\ref{fig:attractor_3} illustrates how attractors decompose the search space.}

Given a distance function $h_{\mathrm{dist}} : \mathcal{S} \times \mathcal{S} \rightarrow \mathbb{R}_{\ge 0}$, attractors are chosen so that the path to any state $s$ in the Open list can be reconstructed by \emph{greedy tracing}. Formally, let $a(s)$ denote the attractor assigned to $s$. Greedy tracing proceeds by repeatedly applying the update
$
s \leftarrow \arg\min_{s' \in Pred(s)} h_{\mathrm{dist}}(s',\, a(s)),
$
until reaching the attractor $a(s)$, after which tracing continues recursively through parent attractors until $start$ is reached. Each such greedy segment corresponds to a portion of the current best path to $s$, which allows all non-attractor states to be safely discarded while still enabling full solution reconstruction. \blue{More details on how attractors are maintained are in the appendix.}

\red{Fig.~\ref{fig:attractor_3} illustrates how attractors decompose the search space into regions. This representation was introduced with Attractor-based Closed List Search, which automatically identifies attractors. More details are in the appendix.}


\section{Method}

In F2F heuristics, computing the heuristic value of a state in one search direction requires comparing it against all states in the opposite frontier to obtain a valid underestimate of the optimal path cost. While this provides highly informative guidance, it also incurs substantial computational overhead. To reduce this cost, we introduce \textit{front-to-attractors} (F2A) heuristics. Although attractors were originally proposed for \red{unidirectional search as a sparse representation of the Closed list}\blue{sparsifying the Closed list}, they also have an important structural property: they form a sparse set of ancestors for the states on the frontier. F2A leverages this structure \red{in the
bidirectional setting }by replacing comparisons against the entire opposite frontier with comparisons against a much smaller, dynamically maintained set of active attractors, denoted $Attrs_D$.

As illustrated in Fig.~2, computing the F2F heuristic for a state $s$ requires pairwise evaluations between $s$ and all states in the opposite frontier (highlighted in yellow). \red{In contrast, the F2A heuristic evaluates $s$ only against the attractors of the opposite direction (highlighted in red), each of which has at least one associated frontier state.} \blue{In contrast, the F2A heuristic evaluates $s$ only against the active attractors in the opposite direction that have at least one associated frontier state (highlighted in red).} \red{This attractor set therefore acts as a compact surrogate for the frontier.} Formally, the F2A heuristic for a state $s$ in direction $D$ is computed as: $h_D(s) = \min_{s' \in Attrs_{OD}} \bigl( h(s, s') + g_{OD}(s') \bigr)$.
Since all least-cost paths discovered so far must pass through their corresponding attractors, F2A preserves optimality guarantees while reducing the number of pairwise evaluations. \footnote{\red{Unlike ACLS, which maintains the full attractor tree back to $start$, F2A keeps only the active attractors—those with at least one associated state in the Open list (colored in Fig. \ref{fig:attractor_3}).}\blue{F2A keeps only the active attractors with associated states in the Open list (colored in Fig~\ref{fig:attractor_3}), rather than the full attractor tree.}}


\begin{figure}
    \centering
    \includegraphics[width=0.8\linewidth]{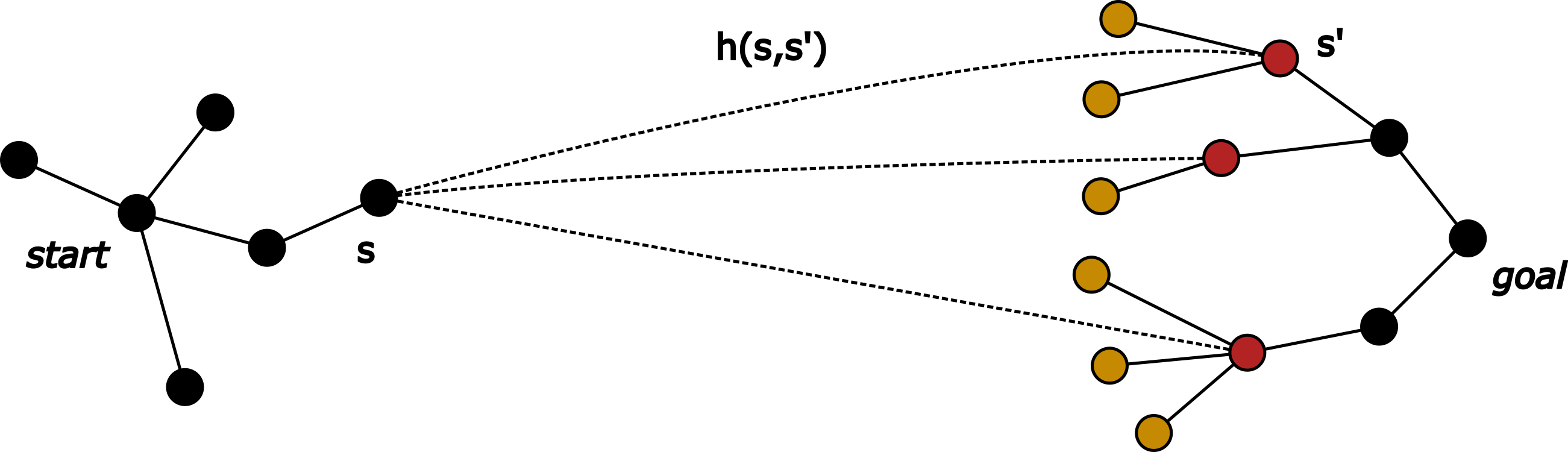}
    \caption{F2F heuristics estimate the distance to the opposite frontier (shown in yellow). F2A heuristics estimate the distance to the opposite attractors (shown in red).}
    \vspace{-0.3cm}
    \label{fig:compare_f2f_f2a}
\end{figure}

\subsection{Algorithm Overview}


Bi-HS with F2A heuristics follows the standard Bi-HS loop used for F2F heuristics, requiring only minor additional bookkeeping to maintain the set of active attractors. The pairwise heuristic $h(s,s')$ must be \blue{bi-}consistent. The overall procedure consists of the following \red{major }steps, as outlined in Alg.~\ref{alg:f2a}. At each iteration, the search:


\begin{enumerate}
    \item \textbf{Selects a search direction:} 
    A common strategy \blue{,which we use,} is to select the direction $D$ with the smaller Open list $Open_D$ (line~\ref{alg:f2a_choose_dir}), as noted by \citet{holte2017mm}. 
    
    \item \textbf{Expands a state:} 
    The state $s$ with the smallest $f_D$-value (higher $g_D$-value for tiebreaking) in $Open_D$ is selected for expansion (lines~\ref{alg:pop_state}--\ref{alg:expand_state}). Its successors are generated, assigned $f$-values, inserted into $Open_D$, and used to update $U$ (lines~\ref{alg:f2a_update_succ_init}--\ref{alg:f2a_update_succ_end}). Each successor’s heuristic is computed once upon generation by comparing the state to attractors in the opposite direction $Attrs_{OD}$ (line~\ref{alg:f2a_update_fvalue}).
    
    \item \textbf{Updates attractors:}  
    When expanding $s$, if a better path to a successor $s'$ is found, the attractor of $s'$ is updated (line~\ref{alg:f2a_update_attractor}). The algorithm attempts to assign $s'$ the same attractor as $s$ to minimize the number of attractors maintained. This is allowed only if greedy tracing from $s'$ toward this attractor leads through $s$ (line~\ref{alg:F2A_greedy_pred}). If this condition holds, $s'$ inherits $s$'s attractor (lines~\ref{alg:F2A_UpdateAttractor_greedypredMatch}--\ref{alg:F2A_assgin_parent_attr}). Otherwise, $s$ becomes the attractor for $s'$ and is inserted into $Attrs_D$ (line~\ref{alg:F2A_UpdateAttractor_storeAttractor}). This update logic follows the same principles \citet{zou2025attractor}\footnote{We also use the same tiebreaking strategy when paths of the same cost are found.}; \blue{details are} omitted from the pseudocode for clarity.

    By assigning attractors in this manner, $Attrs_D$ forms a sparse set of ancestor states through which all current best paths to the frontier states pass. This sparsity makes F2A heuristics much cheaper to compute while still providing strong guidance.

    \item \textbf{Checks for termination:}  
    \red{As in standard Bi-HS, the}\blue{The} search terminates when the current best solution cost $U$ is less than or equal to the lower bound $L$ (optimal path has been found) or if the Open list becomes empty (no solution exists) (line~\ref{alg:f2a_check_termination}).
\end{enumerate}

\begingroup
\begin{algorithm}[!ht]
    \caption{\textsc{Bi-HS with F2A Heuristic} \label{alg:f2a}}
    
    \begin{algorithmic}[1] 
    \Procedure{F2A Bi-HS}{$start, goal$}
        \State $Attrs_F \gets \{start\}, Attrs_B \gets \{goal\}$ \label{alg:f2a_init_begin}
        \State $Open_F \gets \{start\}, Open_B \gets \{goal\}$ \label{alg:f2a_set_open}
        \State $U, L \gets \infty$\label{alg:f2a_init_end} 
        
        \While {\textsc{NotTerminated}($U,L$)} \label{alg:f2a_check_termination}
            \State $D \gets$ \textsc{ChooseDirection()} \label{alg:f2a_choose_dir}
            \State \textsc{Pop} $s$ with smallest $f_D$-value from $Open_D$ \label{alg:pop_state}
            \State \textsc{Expand}($s$, $D$) \label{alg:expand_state}
            \State $L \gets$ \textsc{UpdateLB()} \label{alg:f2a_update_lb}
            \State \textsc{Remove} all $a \in Attrs_D$ that are unassigned \label{alg:remove_attractors}

        \EndWhile
    \EndProcedure
    \Statex       
    \vspace{-0.75em}
    \Procedure{Expand}{$s$, $D$}
        \State Move $s$ from $Open_D$ to $Closed_D$
        \For {$s' \in Succ_D(s)$} \label{alg:f2a_update_succ_init}
            \If{$g_D(s') \geq g_D(s) + c(s, s')$} \label{alg:f2a_check_succ_cost}
                
                \State $g_D(s') = g_D(s) + c(s, s')$ \label{alg:f2a_update_succ_cost}
                \State \textsc{UpdateAttractor}$(s, s',D)$ \label{alg:f2a_update_attractor}
                \If{$s' \in Open_D \cup Closed_D$}
                \State remove $s'$ from $Open_D \cup Closed_D$
                \EndIf
            \EndIf
            
            \If {$s' \notin Open_D$}

                \State $h_D(s') = \min_{s'' \in Attrs_{OD}}(h(s',s'')+g_{OD}(s''))$ \label{alg:f2a_update_fvalue} \blue{\Comment{If not yet computed}}

                \State \textsc{Insert} $s'$ in  $Open_D$ with $g_D(s')+h_D(s')$
            \EndIf
            \If{$s' \in Open_{OD}$} \label{alg:U_update_condition}
                \State $U \gets \min(U,g_F(s')+g_B(s'))$ \label{alg:update_mincost}
            \EndIf
        \EndFor \label{alg:f2a_update_succ_end}
    \EndProcedure
    \Statex
    \vspace{-0.75em}
    \Procedure{UpdateAttractor}{$s,s',D$} \label{alg:F2A_UpdateAttractor}
    
    \State $s^{''}_{min} \gets \argmin_{s^{''} \in Pred(s')}$ $\mathbf{h_{dist}}(s^{''},s.attractor)$ \label{alg:F2A_greedy_pred}
    \State $s'.attractor \gets s$ 
    \If {$s^{''}_{min} = s$} \label{alg:F2A_UpdateAttractor_greedypredMatch}
        \State $s'.attractor \gets s.attractor$ \label{alg:F2A_assgin_parent_attr}
    \ElsIf{$s$ not in $Attrs_D$}
    \State \textsc{Insert} $s$ in $Attrs_D$ \label{alg:F2A_UpdateAttractor_storeAttractor}
    \EndIf
    \EndProcedure

    \end{algorithmic}
\end{algorithm}
\endgroup
\vspace{-0.3cm}

\subsection{Theoretical Properties}



Our proof follows the same structure as \citet{holte2017mm}. Here we provide a high level \red{discussion}\blue{sketch}. A more detailed proof can be found in the appendix.

\begin{definition} \label{def:sij}
For any optimal path $P=s_0,s_1,...s_n$ from $start$ $(s_0)$ to $goal$ $(s_n)$, let $i$ be the largest index such that $s_k \in Closed_F$ $\forall k\in[0,i-1]$, and let $j$ be the smallest index such that $s_k \in Closed_B$ $\forall k\in[j+1,n]$. We say $P$ ``has not been found" if $i < j$ and $P$ ``has been found" otherwise.
\end{definition}

\begin{theorem} \label{theorem:sketch}
If, at the beginning of an iteration of Alg.~\ref{alg:f2a}, there exists an optimal path $P$ from $start$ to $goal$ that has not been found, then $L \le C^*$.
\end{theorem}

\begin{proofsketch}
If there exists an optimal path that has not been found, there will be a node in $P$, $s_i$, in $Open_F$ with $g_F(s_i)=c(start,s_i)$ and a node in $P$, $s_j$, in $Open_B$ with $g_B(s_j)=c(s_j,goal)$. Then $C^* = g_F(s_i)+c(s_i,s_j)+g_B(s_j) \ge g_F(s_i)+h(s_i,s_j)+g_B(s_j)$ (admissibility). With the F2A heuristic, $f_F(s) = g_F(s) + \min_{a \in Attr_B}\bigl(h(s,a)+g_B(a)\bigr)$. By construction of attractors, for every $s' \in OPEN_B$ there exists $a \in Attr_B$ with $g_B(s')=c(s',a)+g_B(a)$. Since $h(s,a) \le h(s,s')+c(s',a)$ (consistency), $h(s,a) + g_B(a) \le h(s,s') + g_B(s')$. Taking minima over attractors yields $f_F(s) \le g_F(s) + \min_{s' \in OPEN_B} \bigl(h(s,s') + g_B(s')\bigr)$. Therefore, $L=\min_{s \in OPEN_F}f_F(s) \le \min_{s \in OPEN_F,\, s' \in OPEN_B} \bigl(g_F(s) + h(s,s') + g_B(s')\bigr) \le g_F(s_i) + h(s_i,s_j) + g_B(s_j) \le C^*.$   
\end{proofsketch}

\begin{theorem}
If there exists a path from $start$ to $goal$, Alg.~\ref{alg:f2a} will return $U=C^*$ (an optimal path) upon termination.
\end{theorem}

\begin{proofsketch}
Theorem \ref{theorem:sketch} implies that $L \le C^*$ until all optimal paths have been found. And since $U > C^*$ until the first optimal path from $start$ to $goal$ is found (due to the process by which $U$ is updated), at which point $U=C^*$, this means that if a path does exist, when Alg.~\ref{alg:f2a} terminates $U=C^*$. 
\end{proofsketch}

This implies that when Alg.~\ref{alg:f2a} terminates at least one optimal path has been found, indicating that Bi-HS with F2A heuristics is both complete and optimal.

\section{Practical Considerations}
\paragraph{Degenerating to F2E}
If $goal$ is included in the backward attractor set $Attrs_B$, then for any state $s$ in the forward direction we have $h(s, goal) \leq h(s, s') + g_B(s') \,\, \forall\, s' \in Attrs_B$,
provided that $h$ is consistent. This implies that the F2A heuristic degenerates to F2E when the goal (or start) is an attractor, reducing the informativeness of the heuristic.


\paragraph{Optimizations to preserve heuristic informativeness.}

In general, attractors become less informative as they drift farther from the frontier. To mitigate this, we introduce two optimizations. Both rely on a threshold $\delta$ applied to the difference in $g$-values between an attractor and its assigned states.

\begin{enumerate}
    \item \textbf{New Attractor (NA).}
    When a successor $s'$ inherits the attractor of its parent $s$, we check whether the difference between $g(s')$ and $g(s.\mathrm{attractor})$ exceeds $\delta$. If so, $s$ is assigned as the new attractor for $s'$. This ensures the active attractor set is a good approximation of the search frontier by keeping attractors close.

    \item \textbf{Associated States (AS).}
    For each attractor, we maintain the subset of its assigned frontier states whose $g$-values differ from the attractor's by more than $\delta$. If this set is nonempty, we use the associated states in place of the attractor for heuristic computation. In effect, when an attractor becomes too far from the frontier to serve as an informative surrogate, we fall back to using its frontier states for heuristic evaluation.
\end{enumerate}

\sisetup{
  detect-all,
  table-number-alignment = center, 
  table-text-alignment = center,
  group-separator = {,}           
}
\setlength{\tabcolsep}{4.5pt}

\begin{table*}[ht]
\centering
\scalebox{0.75}{
\begin{tabular}{
  @{} l
  @{\hskip 4pt} S[table-format=5.0] S[table-format=1.3] S[table-format=5.0] 
  @{\hskip 10pt} S[table-format=6.0] S[table-format=3.0] S[table-format=6.0] 
  @{\hskip 10pt} S[table-format=2.0] S[table-format=2.0] S[table-format=4.0] 
  @{\hskip 10pt} S[table-format=4.0] S[table-format=3.0] S[table-format=5.0] @{} 
}
\toprule
\textbf{Planner}
    & \multicolumn{3}{c}{\textbf{DAO}}
    & \multicolumn{3}{c}{\textbf{Maze}} 
    & \multicolumn{3}{c}{\textbf{Pancake Gap-1}}
    & \multicolumn{3}{c}{\textbf{Sliding Tiles}}
    \vspace{-2pt}
                                      \\
\cmidrule(lr){2-4}
\cmidrule(lr){5-7}
\cmidrule(lr){8-10}
\cmidrule(lr){11-13}
\vspace{-2pt}
    & \multicolumn{1}{c}{\textbf{RT} (ms)}
    & \multicolumn{1}{c}{\textbf{Exp ($10^3$)}}
    & \multicolumn{1}{c}{\textbf{HE ($10^3$)}}
    & \multicolumn{1}{c}{\textbf{RT} (ms)}
    & \multicolumn{1}{c}{\textbf{Exp ($10^3$)}}
    & \multicolumn{1}{c}{\textbf{HE ($10^3$)}}
    & \multicolumn{1}{c}{\textbf{RT} (ms)}
    & \multicolumn{1}{c}{\textbf{Exp ($10^3$)}}
    & \multicolumn{1}{c}{\textbf{HE ($10^3$)}}
    & \multicolumn{1}{c}{\textbf{RT} (ms)}
    & \multicolumn{1}{c}{\textbf{Exp ($10^3$)}}
    & \multicolumn{1}{c}{\textbf{HE ($10^3$)}}\\
\midrule

A*
    & 47 & 17 & 17
    & 534 & 185 & 185 
    & 220 & 6.5 & 67    
    & 2745 & 340 & 605
                                          \\
BAE* F2E & 35 & 20 & 21 & 258 & 143 & 143 & 32 & 0.695 & 8 & 1205 & 151 & 296\\

\midrule

VBi-HS F2E
    & 85 & 25 & 25
    & 1100 & 280 & 281
    & 175 & 4.8 & 53
    & 2746 & 307 & 577
                                      \\
VBi-HS F2F
    & 91 & 14 & 2865
    & 1043 & 134 & 34907 
    & 26545 & 0.541 & 26189
    & 118263 & 20 & 92775
                                   \\
VBi-HS F2A (NA)
    & 76 & 16 & 525
    & 778 & 136 & 4897 
    & 778 & 4.8 & 53
    & 19161 & 108 & 75500
                                     \\

VBi-HS F2A (AS) 
    & 119 & 17 & 1273
    & 1148 & 139 & 15275
    & 14846 & 0.583 & 14782
    & 76174 & 20 & 232716
                                       \\

VBi-HS F2A (no opt)
    
    & 76 & 18 & 278
    & 660 & 138 & 3120 
    & 756 & 4.8 & 53
    & 4374 & 307 & 593\\
\midrule

NBS F2E
    
    & 73 & 21 & 21
    & 558 & 147 & 147     
    & 168 & 4.7 & 54
    & 3111 & 290 & 548\\
NBS F2F
    
    & 99 & 16 & 1794
    & 934 & 135 & 15709 
    & 35068 & 0.828 & 35759
    & 344803 & 37 & 332158\\
NBS F2A (NA)
    
    & 85 & 18 & 319
    & 819 & 138 & 3545        
    & 517 & 4.7 & 183
    & 31832 & 94 & 143202\\

NBS F2A (AS)
    & 553 & 21 & 12764
    & 16553 & 146 & 277741     
    & 50756 & 1.1 & 51659
    & 1301147 & 60 & 780339
                                  \\

NBS F2A (no opt)
    
    & 84 & 18 & 319
    & 745 & 138 & 3545     
    & 449 & 4.7 & 104
    & 5964 & 290 & 4306\\

\bottomrule
\end{tabular}
}
\caption{Results for planners across domains. \red{RT is measured in milliseconds; Exp and HE are scaled by $10^3$. }F2A \red{planners }employ optimizations (“no opt” means no optimization) that vary by domain.}
\label{tab:results}
\end{table*}

\begin{figure}[!t]
    \centering
    \includegraphics[width=0.8\linewidth]{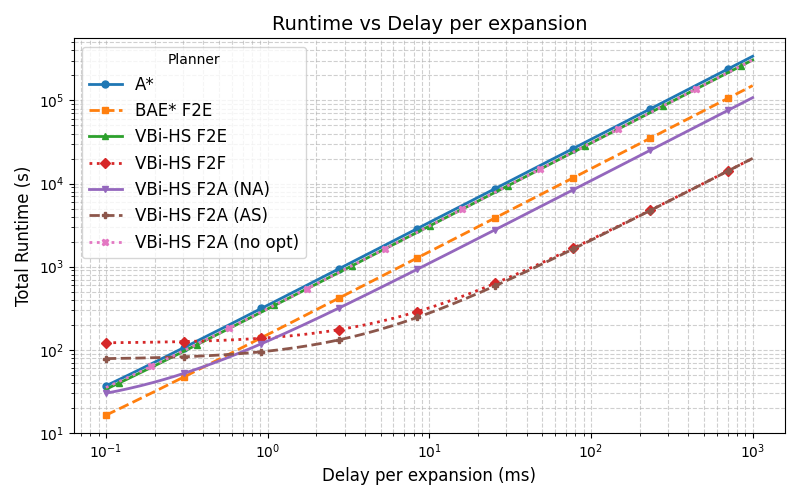}
    \caption{Sliding Tiles with an added delay per expansion.}
    \vspace{-0.3cm}
    \label{fig:delay_plot}
\end{figure}

\section{Experiments \label{sec:f2a_experiments}}

We evaluated F2A on three standard benchmark domains commonly used in Bi-HS studies. \textbf{(1) 2D Grid Pathfinding:} Evaluated on the DAO “brc” and maze maps with five corridor widths \cite{sturtevant2012benchmarks}, using 4-connected unit-cost actions and the Manhattan Distance heuristic. \textbf{(2) 15 Sliding Tiles:} 20 instances (fewest nodes) from \red{Korf’s test set }\cite{korf1985depth}, using the Manhattan Distance heuristic. \textbf{(3) 14-Pancake Puzzle:} 50 random instances \red{evaluated }using the GAP-$1$ heuristic \blue{(GAP heuristic \cite{GAP} but ignoring the smallest pancake.)}

We compared F2A (with and without the NA and AS optimizations) to F2E and F2F using both vanilla Bidirectional Heuristic Search (VBi-HS) and Near-optimal Bidirectional Search (NBS) \cite{NBS}. \blue{VBi-HS denotes the framework used to instantiate F2E, F2F, and F2A (i.e., Alg. 1 with different heuristic formulations), following the ideas of \cite{pohl1969bi} and \cite{BHFFA}.} Since NBS was originally designed for F2E heuristics, we extended it to support F2F and F2A by \red{replacing its heuristic evaluation subroutine with the corresponding formulation.}\blue{changing the heuristic used to compute $f$-values while keeping the same pairwise lower bound.} \red{All other aspects remain the same.} BAE* \cite{BAE} is a state-of-the-art Bi-HS algorithm but assumes \red{a fixed heuristic}\blue{fixed heuristic values assigned at node generation}, making it unclear how to extend it to dynamic heuristics like F2F and F2A \blue{that depend on the evolving opposite frontier}\red{ while preserving informativeness}. However, we include BAE* with F2E heuristics and A* \blue{\cite{A*}} as references. \blue{The same search direction policy is used for the Bi-HS algorithms.} We report results using $\delta = 20$ for both NA and AS on the grid maps, and $\delta = 4$ for Sliding Tiles and Pancake. \red{A discussion on selecting different $\delta$ values \red{can be found}\blue{is} in the appendix.} For each planner we report mean runtime (RT), number of node expansions (Exp), and number of heuristic evaluations (HE), presented in Table~\ref{tab:results}. Our evaluation focuses on how different heuristic classes trade off computation cost against search guidance, rather than on identifying a dominant heuristic. \blue{Additional discussion on the experimental results is included in the appendix.}

\paragraph{Grid Maps}

With VBi-HS, both F2A (no opt) and F2A (NA) achieve the fastest runtimes among all variants. The improvement is primarily driven by a substantial reduction in heuristic evaluations—10.3× fewer on DAO and 11.2× fewer on Maze maps compared to F2F—while still outperforming F2E in terms of expansions (1.4× and 2.0× fewer). Under NBS, the F2A variants also require fewer heuristic evaluations than F2F and produce fewer expansions than F2E. As expected, F2A reduces heuristic evaluations versus F2F while preserving informativeness, yielding fewer expansions than F2E; this trend holds across domains.

\paragraph{Sliding Tiles}

Naive F2A degenerates to F2E in this domain. However, incorporating the NA optimization makes F2A substantially more informative. With VBi-HS, expansions decrease by $2.84\times$ relative to F2E, and by $3.08\times$ under NBS. Heuristic evaluations decrease by $1.23\times$ and $2.32\times$ compared to F2F, for VBi-HS and NBS respectively. The AS optimization further reduces expansions but incurs a much higher number of heuristic evaluations.

To better understand when F2A heuristics provide benefits, we conducted an additional experiment in which an artificial delay was added to every state expansion. The resulting runtime curves, shown in Fig.~\ref{fig:delay_plot}, reveal a performance window \red{between approximately $10^5$ and $10^6$ ns delay per expansion where F2A (NA) achieves the best overall runtime across all planners.}\blue{after $0.6$ ms where F2A (AS) achieves the best runtime across all planners.} This study suggests the existence of domains in which the trade-off between heuristic computation cost and expansion cost aligns favorably for F2A, enabling it to outperform F2E, F2F, and A* in total runtime.

\paragraph{Pancake}
F2A also degenerates to F2E in the Pancake domain. The AS optimization substantially increases the informativeness of F2A. With VBi-HS, expansions decrease by $8.23\times$ relative to F2E, and by $4.21\times$ under NBS.

\section{Conclusion}



We introduced front-to-attractors (F2A), a heuristic for bidirectional search that preserves much of the guidance of front-to-front (F2F) while significantly reducing computational cost. By estimating the distance to a set of attractors rather than all frontier states, F2A requires minimal overhead and integrates easily into existing algorithms. Across multiple domains, it reduced heuristic evaluations by up to $11.2\times$ over F2F and achieved $4.8\times$ fewer expansions than F2E, while maintaining competitive runtimes, making it a practical approach for scalable bidirectional search.

\section{Acknowledgments}
The research was supported by the National Science Foundation by grant IIS-2328671. The views and conclusions in this document are those of the authors and should not be interpreted as representing the official policies, either expressed or implied, of the sponsoring organizations, agencies, or the U.S. government.

\bibliography{aaai2026}

\newpage
\section{Appendix}
\subsection{Delta Values}
We experimented with a range of $\delta$ values and observed that performance is domain dependent. In grid pathfinding, delta values in the range of 20–100 produced similar results, whereas in Sliding Tiles and Pancake the choice of delta was more sensitive. Empirically, choosing delta as a fraction (typically less than one-half) of the solution length yielded good trade-offs, while larger values often caused degeneration toward F2E. These strategies are intended to balance heuristic informativeness against heuristic computation cost, and this balance is inherently domain-specific. \blue{However, the qualitative conclusion remains unchanged across the tested range: smaller delta moves F2A toward F2F-like behavior with more heuristic evaluations, while larger delta reduces the effect of the optimization, with behavior converging to naive F2A, which is more F2E-like in the presence of degeneration. Sliding Tiles and Pancake are more delta-sensitive than grid pathfinding because they have exponential state-space growth and short solution lengths. This is especially pronounced in Pancake due to its higher branching factor.}

\subsection{Degenerating to F2E}
For the Sliding Tiles and Pancake domains we observed degeneration to F2E. \red{This occurs because the state space grows exponentially, making it likely that $start$ or $goal$ becomes an attractor with an assigned frontier state that is never expanded, causing them to persist indefinitely in $Attrs_D$.}\blue{ In permutation domains with exponential state spaces, states close to the start may lie on branches that are never expanded, causing stale attractors near the start/goal to remain and reducing heuristic informativeness.  The NA and AS optimizations are designed to prevent F2A from degenerating to F2E by ensuring that the representatives used in heuristic computation remain close to the frontier.}

\blue{NA does not explicitly remove such attractors, but rather re-associates states to a closer attractor when their parent is expanded. A stale attractor is removed only after all associated states are expanded (and their successors re-associated with closer attractors). As a result, NA is less effective in the Pancake domain, due to its higher branching factor (14 vs. 4), which increases the likelihood of persistent low-g frontier states that are never expanded.}\red{Compared to the Sliding Tiles, this effect is amplified in Pancake by the Pancake puzzle’s larger branching factor (14 versus 4), which increases the probability that distant attractors remain linked to unexpanded nodes, thereby limiting the effectiveness of the NA strategy. The AS strategy overcomes this issue by directly leveraging the associated frontier states during heuristic evaluation, allowing it to maintain informative guidance even in this high-branching, combinatorial domain.}

\blue{AS mitigates this by explicitly disabling attractors far from the frontier and directly using their associated frontier states instead for heuristic calculation, thereby preserving heuristic quality. However, this increases heuristic evaluations. The experimental results reflect this trade-off: in Sliding Tiles, naive F2A matches F2E in expansions, while NA reduces expansions; in Pancake, both naive F2A and NA behave like F2E, suggesting in this domain NA does not prevent degeneration, whereas AS significantly reduces expansions (8.2$\times$ in VBi-HS and 4.3$\times$ in NBS) but incurs higher heuristic computation (278.9$\times$ in VBi-HS and 496.7$\times$ in NBS).}

\subsection{Additional Results}

\blue{Below we include the runtime plots for adding an artificial delay to every state expansion in the Sliding Tiles domain. Fig.~\ref{fig:delay_plot_VBi-HS} shows the runtime curves for VBi-HS while Fig.~\ref{fig:delay_plot_NBS} shows the runtime curves for NBS. As shown in both figures, there is a performance window where an F2A variant has the best runtime across all planners (approximately after $0.6$ ms for VBi-HS (AS) and approximately between $0.6$ and $5$ ms for NBS (NA)).}
\begin{figure}[!t]
    \centering
    \includegraphics[width=\linewidth]{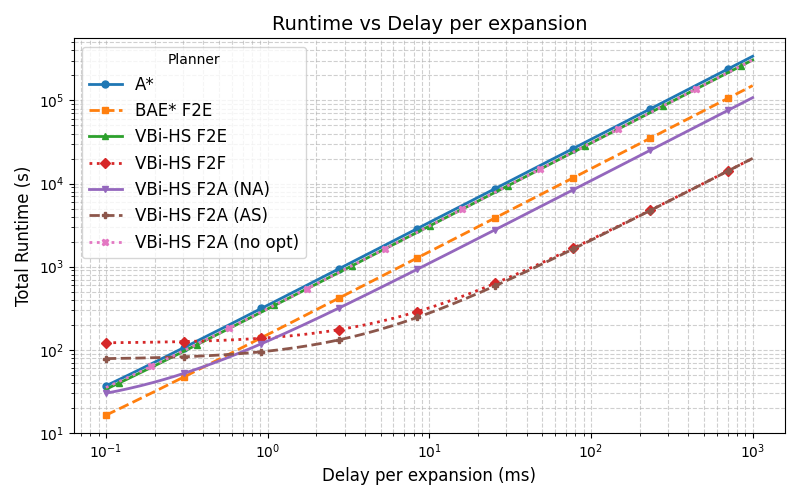}
    \caption{Sliding Tiles with an added delay per expansion (VBi-HS).}
    \label{fig:delay_plot_VBi-HS}
\end{figure}
\begin{figure}[!t]
    \centering
    \includegraphics[width=\linewidth]{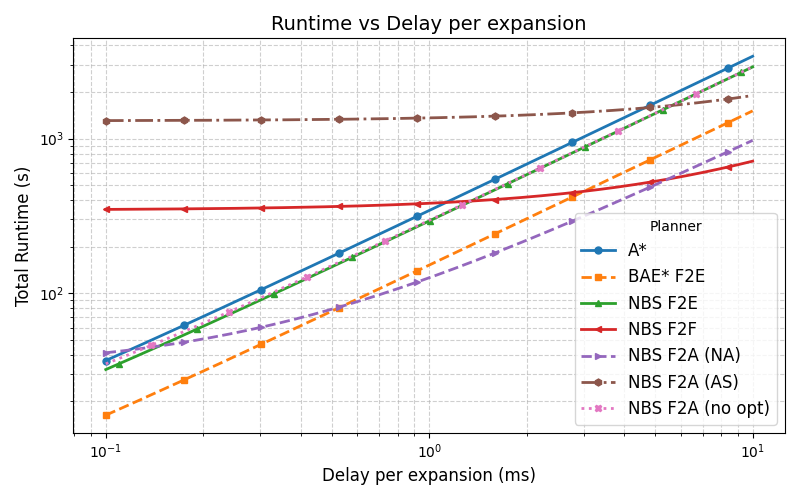}
    \caption{Sliding Tiles with an added delay per expansion (NBS).}
    \vspace{-0.3cm}
    \label{fig:delay_plot_NBS}
\end{figure}

\subsection{Applying F2A Heuristics to Bi-HS Algorithms}
\blue{F2A heuristics are dynamic and change based on the opposite search frontier similarly to F2F heuristics. Because of this, F2A heuristics can be easily adapted to Bi-HS algorithms that support F2F heuristics. However, it is difficult to adapt F2A heuristics to Bi-HS algorithms that do not support F2F heuristics. One such example is BAE*. BAE* cannot directly support F2F or F2A heuristics because it assumes fixed heuristic values assigned at node generation, whereas F2F/F2A values are dynamic and depend on the evolving opposite frontier. In BAE*, both forward and backward heuristic values must be computed when a node is generated and remain fixed thereafter; these values are then used in its error estimates and lower-bound maintenance. F2E heuristics satisfy this requirement. In contrast, for a node generated in one direction, its F2F/F2A value with respect to the opposite direction depends on future changes to the current frontier. We could not design a BAE* adaptation that both retains the benefits of F2F/F2A and preserves BAE*’s optimality assumptions.}

\blue{Other Bi-HS frameworks such as NBS and VBi-HS do not have the same fixed heuristic value requirement in this form, and so we adapted them to support F2F and F2A heuristics.}

\subsection{Maintaining Attractors}
In Attractor Based Closed List Search (ACLS) \cite{zou2025attractor}, attractors are maintained such that at any point during the search the optimal path from the start to any state in the Open list can be reconstructed through a series of greedy traces. The high level idea is that during every expansion the search determines if the state $s$ that is being expanded is required to reconstruct the current best path to any of its successors $s'$. If it is, $s$ is stored as an attractor.

When expanding $s$, if a better path to one of its successors $s'$ is discovered, we must be able to reconstruct this new path to $s'$. To minimize the number of attractors saved, we attempt to assign $s'$ the same attractor as $s$. For $s'$ to inherit $s$'s attractor, it is necessary that $s'$'s greedy tracing through this attractor leads to $s$. If this condition is met, $s$ can be deleted without affecting the path reconstruction. However, in the case of using attractors for the F2A heuristic states are not deleted, as memory is not the constraint and states can be potentially re-expanded. If this condition is not met, then $s$ itself is assigned as the attractor for $s'$, in which case it is saved with its parent attractor set to $s.attractor$. By assigning $s'$'s attractor in this way, the best path to $s'$ is always recoverable.

Attractor maintenance relies on a data structure called $Attrs_D$. This structure maintains the list of attractors and a counter representing the number of states assigned to each attractor. Unlike ACLS, the F2A heuristic only relies on the set of attractors that have at least one associated state in the Open list \blue{(i.e., the active attractors)}. As a result, the counter only keeps track of how many states in $Open_D$ are assigned to it, but not child attractors. These counters are updated whenever states are assigned, reassigned, or removed from $Open_D$. When an attractor no longer has any frontier states associated with it (its counter reaches zero), it is removed from $Attrs_D$.

When multiple paths of the same cost lead to a state $s$, we tiebreak by selecting the path that assigns an attractor farthest from $s$. Here, ``farther away" is quantified as the attractor with a higher $\mathbf{h_{dist}}$ value relative to $s$. For example, consider two parents $s'$ and $s''$, both sharing the same attractor $t$. If $s'$ is the greedy predecessor of $s$ based on $t$ (due to deterministic tiebreaking), but $s''$ was expanded first, $s''$ might be unnecessarily added to the $Attrs$ structure and assigned as $s$'s attractor instead of $t$. To prevent such scenarios, we tiebreak by favoring the path with the attractor that is farther away, such that $t$ will be assigned as $s$'s attractor once $s'$ is expanded. Alternatively, it is possible to bypass the tiebreaking routine altogether at the expense of saving more attractors. 

\subsection{Proof of Optimality}

In this section we will prove the optimality of Alg. 1. We will follow the same structure as \citet{holte2017mm}, with the notable difference being Theorem \ref{theorem:L<=C*}. Let $C^*=c(start,goal)$ be the optimal cost.

\begin{definition} \label{def:sij_appendix}
For any optimal path $P=s_0,s_1,...s_n$ from $start$ $(s_0)$ to $goal$ $(s_n)$, let $i$ be the largest index such that $s_k \in Closed_F$ $\forall k\in[0,i-1]$, and let $j$ be the smallest index such that $s_k \in Closed_B$ $\forall k\in[j+1,n]$. We say $P$ ``has not been found" if $i < j$ and $P$ ``has been found" otherwise.
\end{definition}

\begin{definition}
Node $n$ is "permanently closed" in the forward direction if $n \in Closed_F$ and $g_F(n) = c(start,n)$. Likewise, $n$ is permanently closed in the backward search direction if $n \in Closed_B$ and $g_B(n) = c(n,goal)$.
\end{definition}

The name "permanently closed" is based on the following lemma.

\begin{lemma} \label{lemma:permanently_closed}
If node $n$ is permanently closed in a particular search direction at the start of some iteration, it will be permanently closed in that direction at the start of all subsequent iterations.
\end{lemma}

\begin{proof}
This proof is for the forward search, the proof for the backward search is analogous. Alg. 1 does not change $g_F(n)$ while $n\in Closed_F$, so $n$ can only stop being permanently closed by being removed from $Closed_F$. This is possible but only if a strictly cheaper path to $n$ is found, which is not possible since $g_F(n) = c(start,n)$ for a node permanently closed in the forward direction. Therefore, once $n$ is permanently closed in the forward direction it will remain so.
\end{proof}

\begin{lemma} \label{lemma:permanently_closed_optimal_succ}
Let $P=s_0,s_1,...s_n$ be an optimal path from $start=s_0$ to any state $s_n$. If $s_n$ is not permanently closed in the forward direction and either $n=0$ or $n>0$ and $s_{n-1}$ is permanently closed in the forward direction, then $s_n\in Open_F$ and $g_F(s_n)=c(start,s_n)$. Analogously, let $P=s_0,s_1,...s_n$ be an optimal path from any state $s_0$ to $goal=s_n$. If $s_0$ is not permanently closed in the backward direction and either $n=0$ or $n>0$ and $s_1$ is permanently closed in the backward direction, then $s_0\in Open_B$ and $g_B(s_0)=c(s_0,goal)$.
\end{lemma}

\begin{proof}
This proof is for the forward search, the proof for the backward search is analogous. If $n=0$, $s_0=start$ has not been closed in the forward direction and the lemma is true because line \ref{alg:f2a_set_open} put $start \in Open_F$ with $g_F(start) = c(start,start)=0$. Suppose $n>0$. When $s_{n-1}$ was expanded and permanently closed in the forward direction $s_n$ was generated via an optimal path (in lines \ref{alg:f2a_check_succ_cost} and \ref{alg:f2a_update_succ_cost}), $g_F(s_n)=g_F(s)+c(s,s')=c(start,s_{n-1})+c(s_{n-1},s_n)=c(start,s_n)$. $s_n$ cannot have been permanently closed in the forward direction at that time because if it was, it still would be (Lemma \ref{lemma:permanently_closed}). If $s_n\in Closed_F \cup Open_F$ at that time with a suboptimal $g$-value, then it would have been removed from $Closed_F \cup Open_F$ and added to $Open_F$ with $g_F(s_n)=c(start,s_n)$. If $s_n\notin Closed_F \cup Open_F$ at that time, it would likewise have been added to $Open_F$ with $g_F(s_n)=c(start,s_n)$. Finally, if $s_n \in Open_F$ at that time with $g_F(s_n)=c(start,s_n)$, it would have remained so. Therefore, no matter what $s_n$'s status was at the time $s_{n-1}$ was expanded to become permanently closed in the forward direction, at the end of that iteration $s_n \in Open_F$ and $g_F(s_n)=c(start,s_n)$. In subsequent iterations $g_F(s_n)$ cannot have changed, since that only happens if a strictly cheaper path to $s_n$ is found, which is impossible. It also cannot have been closed, since if that had happened it would now be permanently closed.
\end{proof}

\begin{lemma} \label{lemma:nFB}
Let $P=s_0,s_1,...s_n$ be an optimal path from $start$ $(s_0)$ to any state $s_n$. If $s_n$ is not permanently closed in the forward direction then there exists an $i$ $(0 \le i \le n)$ such that $s_i \in Open_F$ and $g_F(s_i)=c(start,s_i)$. Let $i_{min}$ be the smallest such $i$ and define $n_F$ (for path $P$) to be $s_{i_{min}}$. Analogously, let $P=s_0,s_1,...s_n$ be an optimal path from any state $s_0$ to $goal=s_n$. If $s_0$ is not permanently closed in the backward direction then there exists a $j$ $(0 \le j \le n)$ such that $s_j \in Open_B$ and $g_B(s_j)=c(s_j,goal)$. Let $j_{max}$ be the largest such $j$ and define $n_B$ (for path $P$) to be $s_{j_{max}}$.
\end{lemma}

\begin{proof}
This proof is for the forward search, the proof for the backward search is analogous. If $start \notin Closed_F$, then $i=0$ has the required properties. Suppose $start \in Closed_F$. Let $k$ $(0 \le k \le n)$ be the largest index such that $s_k$ is permanently closed. Such a $k$ must exist because $start$ $(k=0)$ is permanently closed. By Lemma \ref{lemma:permanently_closed_optimal_succ} $s_{k+1} \in Open_F$ and $g_F(s_{k+1})=c(start,s_{k+1})$. Therefore $i=k+1$ has the required properties. 
\end{proof}

The following lemma shows that for an optimal path $P$ from $start$ to $goal$, $n_F$ and $n_B$ for $P$ are exactly the states $s_i$ and $s_j$ for $P$ defined in Definition \ref{def:sij_appendix}.

\begin{lemma} \label{lemma:nFB_exist}
Let $P=s_0,s_1,...s_n$ be an optimal path from $start$ $(s_0)$ to $goal$ $(s_n)$ that has not been found. Then $n_F$ and $n_B$, as defined in Lemma \ref{lemma:nFB}, both exist for $P$ and $n_F=s_i$ and $n_B=s_j$, where $s_i$ and $s_j$ are as defined in Definition \ref{def:sij_appendix}.
\end{lemma}

\begin{proof}
Let $i$ and $j$ be as in Definition \ref{def:sij_appendix}. For the forward search, $s_0,s_1,...s_i$ is an optimal path from $start$ to $s_i$ and $s_i \notin Closed_F$ and therefore is not permanently closed in the forward direction. Therefore, $s_0,s_1,...s_i$ satisfies Lemma \ref{lemma:nFB} for the forward direction and $n_F=s_{i'}$ exists for path $s_0,s_1,...s_i$. Because $s_0,s_1,...s_{i-1}$ are all in $Closed_F$, it must be that $i'=i$. Since $i'$ is the smallest index between 0 and $i$ such that $s_{i'} \in Open_F$ and $g_F(s_{i'}) = c(start, s_{i'})$, it is also the smallest index between 0 and $n$ with these properties, so $s_{i'}$ is also $n_F$ for path $P$. For the backward search the reasoning is analogous.
\end{proof}

\begin{theorem} \label{theorem:L<=C*}
If, at the beginning of an iteration of Alg. 1, there exists an optimal path $P$ from $start$ to $goal$ that has not been found, then $L \le C^*$.
\end{theorem}

\begin{proof}
 Let $n_F=s_i$ and $n_B=s_j$ be as defined in Lemma \ref{lemma:nFB}. Lemma \ref{lemma:nFB_exist} guarantees that $n_F$ and $n_B$ exist for $P$. Because $i<j$, with an admissible pairwise heuristic $C^* = g_F(s_i)+c(s_i,s_j)+g_B(s_j) \ge g_F(s_i)+h(s_i,s_j)+g_B(s_j)$. With the F2A heuristic, $f_F(s) = g_F(s) + \min_{a \in Attr_B}\bigl(h(s,a)+g_B(a)\bigr)$. By construction of attractors, for every $s' \in OPEN_B$ there exists $a \in Attr_B$ with $g_B(s')=c(s',a)+g_B(a)$. With a consistent pairwise heuristic, $h(s,a) \le h(s,s')+c(s',a)$, which implies $h(s,a) + g_B(a) \le h(s,s') + g_B(s')$. Taking minima over attractors yields $f_F(s) \le g_F(s) + \min_{s' \in OPEN_B} \bigl(h(s,s') + g_B(s')\bigr)$. Therefore, $L=\min_{s \in OPEN_F}f_F(s) \le \min_{s \in OPEN_F,\, s' \in OPEN_B} \bigl(g_F(s) + h(s,s') + g_B(s')\bigr) \le g_F(s_i) + h(s_i,s_j) + g_B(s_j) \le C^*.$   
\end{proof}

\begin{lemma} \label{lemma:no_termination_until_found}
If there exists a path from $start$ to $goal$, Alg. 1 will not terminate until at least one optimal path from $start$ to $goal$ has been found.
\end{lemma}

\begin{proof}
Lemma \ref{lemma:nFB_exist} guarantees that $Open_F$ and $Open_B$ are both non-empty as long as there is any optimal path from $start$ to $goal$ that has not been found, so the termination condition for empty Open lists cannot be satisfied until all optimal paths from $start$ to $goal$ have been found. The only other termination condition is $U \le L$. Assume for the purpose of contradiction that this termination condition is satisfied before any optimal path from $start$ to $goal$ has been found. Theorem \ref{theorem:L<=C*} shows that $L \le C^*$ until all optimal paths from $start$ to $goal$ have been found, so for $U \le C^*$ to hold if no optimal paths from $start$ to $goal$ have been found, $U$ must be equal to $C^*$. We will now show that $U=C^*$ implies an optimal path from $start$ to $goal$ has been found, contradicting our assumption, thereby proving the lemma. $U$ is set in line \ref{alg:update_mincost}. On the iteration in which $U$ was set to $C^*$, there must have been a child node generated, $s'$, that satisfied the conditions of line \ref{alg:U_update_condition}, i.e., $s' \in Open_B$ and $g_F(s')+g_B(s')=C^*$. The latter implies $g_F(s')=c(start,s')$ and $g_B(s')=c(s',goal)$, i.e., $s'$ is on an optimal path from $start$ to $goal$ with optimal $g$-values in both directions. This means that all the nodes on the forward and backward generating paths for $s'$ are permanently closed\red{(Using a lemma that was omitted)}. The concatenation of these two paths, with $s'$ in between, is an optimal path from $start$ to $goal$ that was found on the iteration when $U$ was set to $C^*$.
\end{proof}

\begin{lemma} \label{lemma:set_U=C*}
If there exists a path from $start$ to $goal$, let $P=s_0,s_1,...s_n$ be the first optimal path from $start$ $(s_0)$ to $goal$ $(s_n)$ that is found during Alg. 1's execution, and let $n_F=s_i$ and $n_B=s_j$ be as defined in Lemma \ref{lemma:nFB} at the beginning of the iteration on which $P$ is found. Then during that iteration $U$ will be set to $C^*$ in line \ref{alg:update_mincost}.
\end{lemma}

\begin{proof}
Lemma \ref{lemma:no_termination_until_found} guarantees that $P$ exists, and Lemma \ref{lemma:nFB_exist} guarantees that $n_F$ and $n_B$ exist for $P$ at the beginning of the iteration on which it becomes found. One of them must be expanded on this iteration because $P$'s status will not change from "not found" to "found" if $n_F$ remains on $Open_F$ and $n_B$ remains on $Open_B$. We will complete the proof assuming that $n_F$ is expanded. The proof in the case that $n_B$ is expanded is analogous. We will first prove that when $n_F$ is expanded, $n_B$ will be generated as one of its children.

Suppose $n_B$ is not generated as a child of $n_F$ when it is expanded in the forward direction. Then there must exist one or more nodes between them, i.e., $P=start...n_F,t_1...t_k,n_b...goal$ $(k\ge 1)$. In order for $P$ to be "found" at the end of this iteration, it must be the case that $t_i \in Closed_F \forall 1 \le i \le k$. Since the path $start...n_F,t_1$ is an optimal path from $start$ to $t_1$, $t_1 \in Closed_F$ after being generated by $n_F$ means that it had previously been generated via a different optimal path, which implies an optimal path had previously been found from $start$ to all the $t_i$ and to $n_B$. Combining this previously found optimal path from $start$ to $n_B$ with the optimal path found by the backwards search from $n_B$ to $goal$ creates an optimal path from $start$ to $goal$ that had been found prior to $P$. This contradicts the premise that $P$ is the first optimal path from $start$ to $goal$.

Because $n_B$ is a child of $n_F$, $n_B \in Open_B$ and $g_F(n_B)+g_B(n_B) = c(start,n_B)+c(n_B,goal)=C^*$, so $U$ will be set to $C^*$ in line \ref{alg:update_mincost}.
\end{proof}

\begin{theorem}
If there exists a path from $start$ to $goal$, Alg. 1 will return $U=C^*$ (an optimal path) upon termination.
\end{theorem}

\begin{proof}
Lemma \ref{lemma:no_termination_until_found} has shown that if there is a path from $start$ to $goal$, Alg. 1 will not terminate until at least one optimal path from $start$ to $goal$ has been found. And by Lemma \ref{lemma:set_U=C*} when the first optimal path is found $U$ will be set to $C^*$. Thus, when Alg. 1 terminates $U \le C^*$. Since $U \ge C^*$ by construction, when Alg. 1 terminates $U=C^*$.
\end{proof}

\end{document}